\documentclass{article} 
\usepackage{iclr2026_conference,times}

\usepackage{amsmath,amsfonts,bm}

\def\eqref#1{equation~\ref{#1}}

\def\1{\bm{1}}

\DeclareMathAlphabet{\mathsfit}{\encodingdefault}{\sfdefault}{m}{sl}
\SetMathAlphabet{\mathsfit}{bold}{\encodingdefault}{\sfdefault}{bx}{n}

\usepackage{hyperref}
\usepackage{url}
\usepackage{amsmath,amssymb,amsfonts}
\usepackage{graphicx}
\usepackage{textcomp}
\usepackage{xcolor}
\usepackage{comment}
\usepackage{float}
\usepackage{algorithm}
\usepackage{algpseudocode}
\usepackage{amsmath}
\usepackage{amssymb}
\usepackage{multirow}
\usepackage{mathtools}
\usepackage{dsfont}
\usepackage{amsthm}

\newtheorem{theorem}{Theorem}

\usepackage{gensymb}
\usepackage{subcaption}
\usepackage{siunitx}
\usepackage{wrapfig}

\makeatletter
\algrenewcommand\alglinenumber[1]{}
\makeatother

\title{Elbow-based MoE Routing: A Training-Free Inference Time Plugin for Expert Selection}

\author{
  \textbf{Robin Pan}$^1$\thanks{rpan@college.harvard.edu} \quad \textbf{Raymond Liu}$^{1\dagger}$ \quad \textbf{Daniel Fang}$^{1}$\thanks{These authors contributed equally} \quad \textbf{Adelina Andrei}$^2$ \quad \textbf{Rosa Wu}$^1$ \\
  $^1$ John A. Paulson School of Engineering and Applied Sciences, Harvard University \\
  $^2$Department of Mathematics, Harvard University 
}

\iclrfinalcopy 
\begin{document}
\maketitle

\vspace{-10pt}
\begin{abstract}
Mixture-of-Experts (MoE) models enable model scaling while maintaining low inference-time compute by activating only a subset of experts per token. However, conventional routing relies on a fixed top-k selection, forcing the model to spend the same compute regardless of how many experts are relevant. We introduce elbow-based routing, a training-free inference-time modification that dynamically adjusts the number of experts on a per-token basis. Our method examines the sorted router probability distribution and identifies an elbow point that separates high- and low-probability experts. We find that most router distributions exhibit clear inflection points suitable for this strategy, and we show both theoretically and empirically that elbow-based routing preserves expert load balance. Experiments on a state-of-the-art MoE model demonstrate an average latency reduction of 5.3\% while maintaining accuracy across six benchmarks.
\end{abstract}

\section{Introduction}

Mixture-of-Experts (MoE) architectures scale language models efficiently by activating only a subset of experts per token, enabling models to grow in size without proportionally increasing inference cost (\cite{moe}). During both training and inference, MoE model routers activate the top-$k$ experts with the highest logits across all tokens regardless of token. Consequently, some tokens with only a few relevant experts consume unnecessary compute, while tokens requiring more experts may be under-routed.

Prior works in dynamic routing require either training a router, expensive hyperparameter searches, or both. \cite{huang2024harder} demonstrates that some inputs benefit from increased expert capacity, and uses auxiliary losses to train a top-$p$ router based on a threshold probability mass $p$. However, $p$ is a hyperparameter that must be tuned via grid search, which is an expensive process and may be sensitive to biases in the validation data. DynMoE~\cite{dynmoe} treats routing as a multi-label classification problem with each expert as a label and allows top-any, but also requires router training.  This motivates our central question: \textit{Can MoE models dynamically adjust the number of active experts per token to reduce latency without retraining or additional hyperparameter tuning?}

We introduce elbow-based routing, a simple inference-time plugin that adaptively selects $k$ by detecting the elbow point in the router’s sorted probability curve (Fig.~\ref{fig1}), which allows us to reduce compute without significantly affecting accuracy. We evaluate our approach on OLMoE, a state-of-the-art MoE model~\cite{olmoe}. Our contributions are as follows: 

\begin{enumerate}
    \item Our work introduces the first training-free and inference-time plugin for dynamic routing that does not require altering model weights, architecture, or any additional losses.
    \item We conduct an analysis of router probability distributions and find that the vast majority of router probabilities have distinct, sharp elbows and this characteristic is largely independent of input type or router layer.
    \item We conduct a comprehensive empirical evaluation across MMLU, ARC-Easy, ARC-Challenge, HellaSwag, PIQA, and WinoGrande, demonstrating an average latency reduction of 5.3\% while maintaining accuracy and having minimal effect on load balancing.
\end{enumerate}

\newpage
\section{Elbow-Based Routing}

\begin{wrapfigure}{r}{0.5\textwidth}
\vspace{-0.9\baselineskip}
\centering
\includegraphics[width=\linewidth]{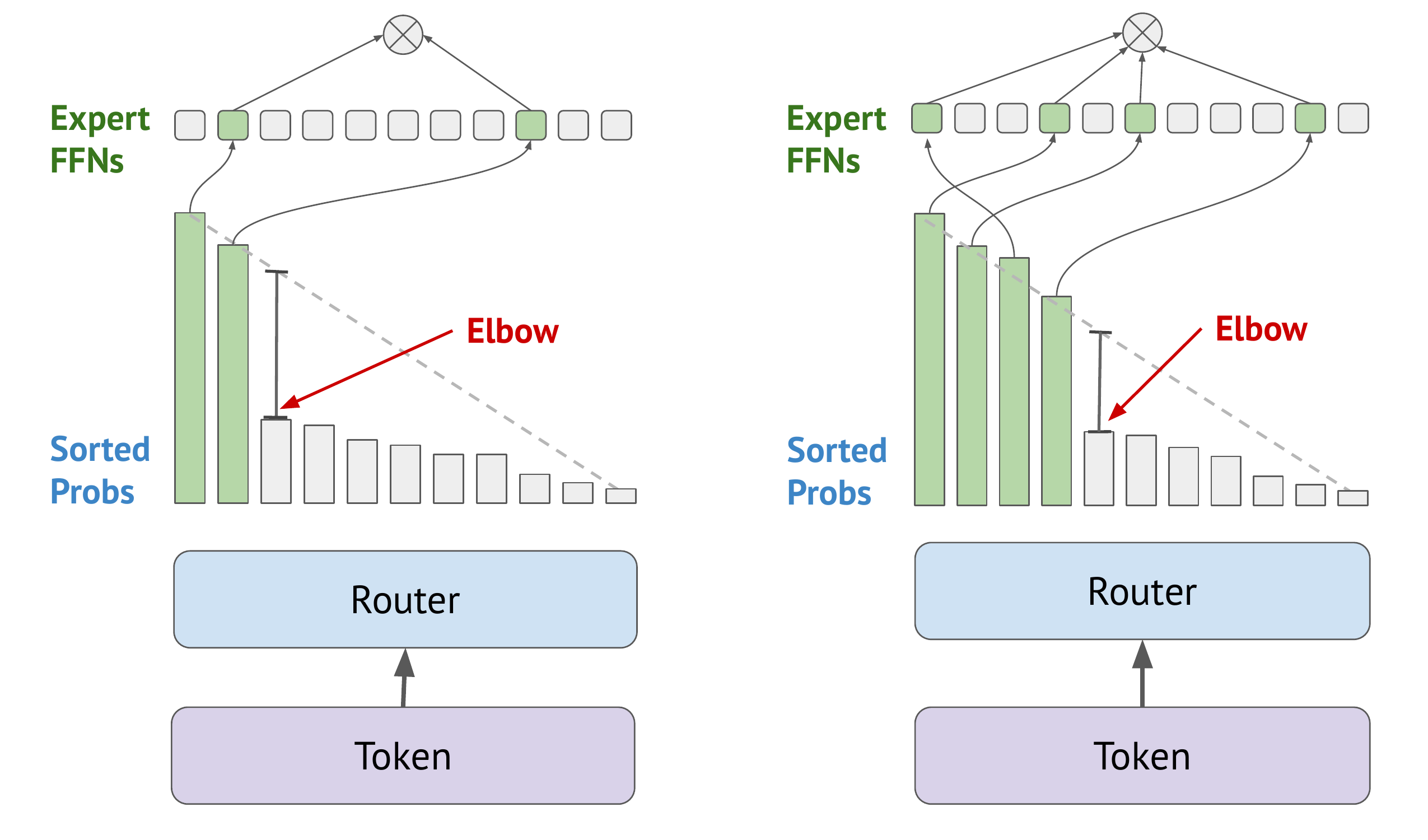}
\caption{\textbf{Elbow-based routing}. Our method adapts to each token's probability curve, which allows us to reduce computation when a clear separation exists between high-confidence and low-confidence experts. Sharper elbows (left) select fewer experts, and more gradual elbows (right) select more experts.}
\label{fig:elbow_routing_examples}
\vspace{-\baselineskip}
\label{fig1}
\end{wrapfigure}

We consider a mixture-of-experts model with $N$ experts and a learned router that produces logits $\ell(x)\in\mathbb{R}^N$ for each token $x$. Applying a softmax yields router probabilities $p(x)=\mathrm{softmax}(\ell(x)),$ which we sort in descending order as $p_{(1)}(x)\ge\dots\ge p_{(N)}(x)$. For each token, we compute an elbow index $e(x)$ on the full sorted probability vector $\{p_{(i)}(x)\}_{i=1}^N$. The elbow is identified using a Kneedle-style~\cite{kneedle} criterion that selects the index of maximum deviation from a reference line after normalizing indices and probabilities to $[0,1]$ (Algorithm~\ref{alg:elbow}). Intuitively, the elbow corresponds to the index at which the rapidly decaying head of the distribution transitions into a relatively flat tail.

\subsection{Capped Expert Selection}

The final number of active experts $k(x)$ is $k(x)=\min\bigl(e(x),K\bigr)$, where $K$ is the number of experts activated in standard top-$K$ routing. Given that $k(x)\le K$ by construction, elbow-based routing is a monotone pruning rule on top of an existing router. It preserves expert ordering, never introduces new experts, and requires no retraining or additional routing hyperparameters beyond the model’s fixed top-$K$ constraint. Elbow-based routing is fast, with a time complexity of $\mathcal{O}(N \log N)$ where $N$, the number of experts, is typically on the order of 10 to 100.

\subsection{Signal--Noise Interpretation}
\captionsetup{aboveskip=2pt, belowskip=-6pt}
\begin{wrapfigure}{r}{0.54\textwidth}
  \centering
  \vspace{-\baselineskip}
  \includegraphics[width=\linewidth]{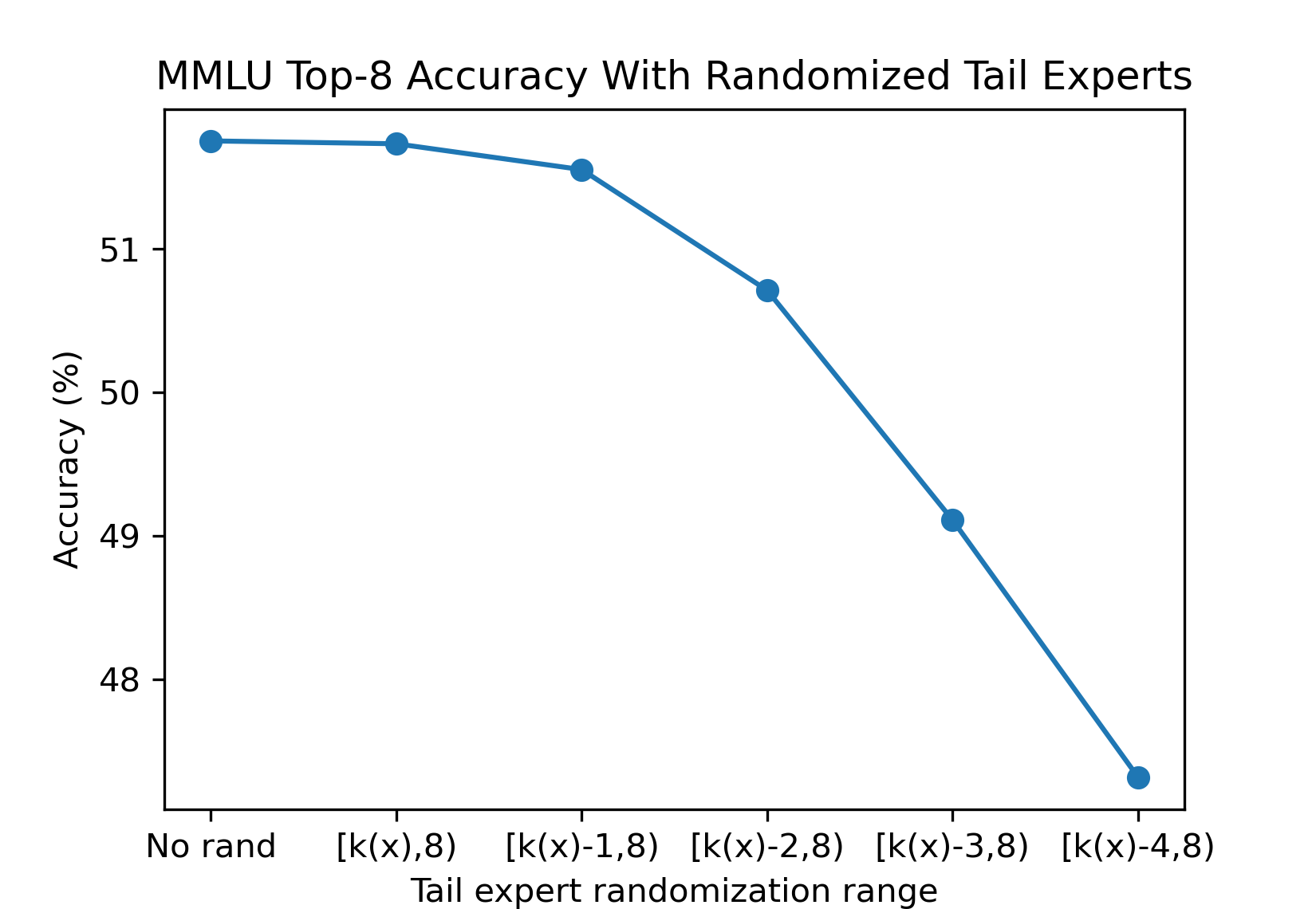}                  
  \caption{\textbf{Tail randomization confirm signal-noise interpretation.} Randomly exchanging tail experts with indices in $[k(x),8)$ with experts in $[k(x),64)$ does not affect accuracy. Accuracy drops when randomizing experts in $[0,k(x))$.} 
  \vspace{-\baselineskip}
  \label{fig:randomization}
\end{wrapfigure}

The effectiveness of elbow-based routing can be understood through a signal--noise interpretation of the router logits. An input vector $x$ to a router may be decomposed as $x = x_{\parallel} + x_{\perp}$, where $x_{\parallel}$ denotes components aligned with routing-relevant directions, and $x_{\perp}$ captures residual variation orthogonal to these directions. Relevant 'signal' experts aligned with $x_{\parallel}$ receive large, structured logits that form the head of the sorted routing distribution, while less relevant 'noise' experts influenced mainly by $x_{\perp}$ exhibit small, unstructured variations, yielding a flat tail.

Under this view, the sorted router distribution naturally exhibits a transition between a small set of high-confidence experts and a diffuse tail. The elbow identifies this transition point, providing a token-specific estimate of how many experts are meaningfully engaged by the router. When such a transition is pronounced, the elbow corresponds to a sharp change in slope, which is reflected geometrically by a large elbow angle. In Section~\ref{sec:elbow_analysis} we quantify this effect by measuring elbow angles across layers and tokens, and we show that the vast majority of routing distributions exhibit sharp elbows, consistent with this signal--noise interpretation.

We empirically confirm this interpretation by evaluating model accuracy when randomizing post-elbow experts. Given a router probability curve with $k(x) < 8$, we replace post-elbow experts (ranks $k(x)$ to 8) with randomly chosen experts of rank $\ge k(x)$. As shown in Fig.~\ref{fig:randomization}, this tail randomization matches the baseline top-8 accuracy with no randomization on MMLU. However, accuracy decreases when replacing pre-elbow experts (ranks 1 to $k(x)$) in the same way. This confirms that the elbow marks a practical signal–noise boundary for routing.

\subsection{Elbow Analysis and Characterization}
\label{sec:elbow_analysis}

We analyzed the `elbow-ness' of router probabilities from OlMoE routers across multiple benchmarks. To characterize `elbow-ness' we calculate an elbow angle $\theta$ between the first normalized point $(0,0)$, the elbow point $(x'_e,p'_e)$, and the last normalized point $(1,1)$.
\begin{align}
\theta &= \arccos\left(
\frac{\mathbf{v}_1 \cdot \mathbf{v}_2}
{\|\mathbf{v}_1\|\,\|\mathbf{v}_2\|}
\right), \\
\text{where}\quad
\mathbf{v}_1 &=
\begin{pmatrix}
-x'_e\\
-p'_e
\end{pmatrix},
\qquad
\mathbf{v}_2 =
\begin{pmatrix}
1-x'_e\\
1-p'_e
\end{pmatrix}.
\end{align}
We analyzed over 2 million router probability curves that were produced using samples from MMLU, ARC-Easy, and ARC-Challenge and found that 99.7\% of them possess clear, sharp elbows ($\leq 135 \degree$) (Fig. ~\ref{fig:elbow_analysis}a).

\begin{figure}[H]
\centerline{\includegraphics[width=\textwidth]{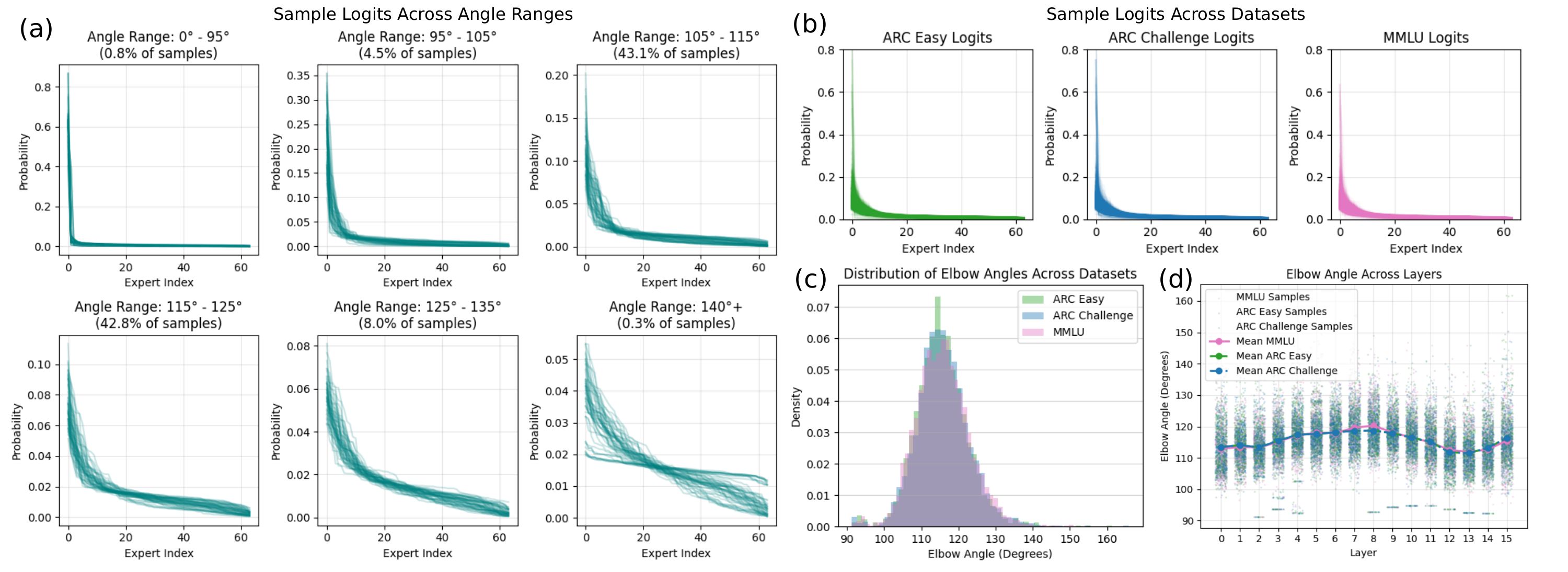}}
\caption{\textbf{Analysis of router probability distributions}. (a) Elbow angle distribution showing 99.7\% of cases exhibit sharp elbows ($\leq \ang{135}$). (b) Mean sorted router probabilities across three benchmark datasets demonstrate consistent router logit patterns. (c) Elbow angle distributions by dataset show minimal variation across datasets. (d) Mean elbow angles remain stable across all 16 router layers for all three datasets, indicating that elbow structure is a robust, layer-invariant property.}
\label{fig:elbow_analysis}
\end{figure}

Separating these probabilities by dataset, we see that router probabilities and elbow angles of these probabilities are nearly identical across datasets (Fig.~\ref{fig:elbow_analysis}b-c). Additionally, mean elbow angles are consistent across router layers and across datasets (Fig.~\ref{fig:elbow_analysis}d). Because the MMLU dataset is categorized by subject, we analyzed router elbow angles and index by subject and found very similar elbow angle and elbow index distributions across all subjects (Supplementary Fig.~\ref{supp1}).

\subsection{Load-Balance Guarantees}
\label{sec:load_balance_guarantees}

For each token, elbow-based routing selects  a subset of the experts chosen by standard top-$K$ routing, reducing the number of active experts without introducing new paths or changing relative ordering. In Appendix~\ref{app:load_balance}, we formalize the effect of this pruning on expert load and show that the induced changes to the normalized per-expert utilization distribution are bounded.

Let $L_i^{(\mathrm{top}\text{-}K)}$ and $L_i^{(\mathrm{elb})}$ denote the expected load of expert $i$ under top-$K$ and elbow-based routing, respectively, and let $\delta = 1 - \mathbb{E}[k(x)] / K$ denote the expected fraction of pruned expert assignments. We prove that the normalized expert utilization distributions
$q^{(\mathrm{top}\text{-}K)}$ and $q^{(\mathrm{elb})}$ satisfy
\[
\left\lVert q^{(\mathrm{elb})} - q^{(\mathrm{top}\text{-}K)} \right\rVert_1
\;\le\;
\frac{2\delta}{1 - \delta}.
\]

\begin{wraptable}{r}{0.48\textwidth}
\vspace{-\baselineskip}
\centering
\scriptsize
\setlength{\tabcolsep}{3pt}
\caption{Layer-wise Load Balancing Analysis}
\begin{tabular}{@{}l l l l@{}}
\hline
Layer & $\ell1$ norm & \% Change & \% Change \\
      &  & in Top-1 Share & in CV\\
\hline
1  & 0.049 & -1.81 & -5.80 \\
2  & 0.055 & -3.71 & -8.83 \\
3  & 0.043 & -4.45 & -4.73 \\
4  & 0.042 & -2.53 & -3.91 \\
5  & 0.047 & -3.85 & -4.29 \\
6  & 0.030 & -2.48 & -2.16 \\
7  & 0.020 & -2.31 & -2.04 \\
8  & 0.019 & -2.08 & -1.79 \\
9  & 0.033 & -2.61 & -1.74 \\
10 & 0.023 & -0.48 & -0.81 \\
11 & 0.020 & +0.38  & +0.05  \\
12 & 0.032 & -1.71 & -1.51 \\
13 & 0.036 & -1.01 & -1.53 \\
14 & 0.048 & -2.95 & -3.20 \\
15 & 0.032 & -1.60 & -1.38 \\
16 & 0.020 & -1.42 & -1.55 \\
\textbf{Mean} & \textbf{0.034} & \textbf{-2.16} & \textbf{-2.83} \\
\hline
\end{tabular}
\label{loadbalance}
\vspace{-0.5\baselineskip}
\end{wraptable}

We measure expert load usage averaged across six benchmarks: MMLU, ARC-Easy, ARC-Challenge, HellaSwag, PIQA, and WinoGrande. Across all benchmarks, elbow-based routing consistently selects around $\mathbb{E}[k(x)] \approx 7.6$ experts per token across all layers as seen in Table~\ref{tab:eval_results}, which corresponds to an average pruning rate $\delta = 1 - \frac{\mathbb{E}[k(x)]}{K} =0.05$. Empirically, we find that changes to load balance are much smaller than the worst-case theoretical bounds at this $\delta$. For each layer, we calculate the $\ell_1$ distance between the normalized expert utilization distributions under top-$K$ and elbow-based-$K$ routing, averaged across all six benchmarks. Table~\ref{loadbalance} reports a mean change of $0.034$ in the $\ell_1$ norm,  which suggests that only $1.7\%$ of the total normalized utilization mass is redistributed across experts. We calculate $\%$ Change in Top-1 Share as $\frac{\max(q^{(top-K)})-\max(q^{(elb)})}{\max(q^{(top-K)}}\cdot100$ and find the maximum expert load increases by $2.16\%$ on average. $\%$ Change in CV is calculated as $\frac{CV^{top-K}-CV^{elb}}{CV^{top-K}}\cdot100$, and CV increases by $2.83\%$. For all three metrics, empirical evidence is significantly smaller than theoretical guarantees in ~\ref{app:load_balance}.

\begingroup
\setlength{\intextsep}{2pt}        
\setlength{\columnsep}{10pt}       
\setlength{\abovecaptionskip}{1pt} 
\setlength{\belowcaptionskip}{0pt}

\begin{wraptable}{l}{0.6\textwidth}
\vspace{-0.2\baselineskip}
\centering
\scriptsize
\setlength{\tabcolsep}{5pt}
\caption{Evaluation Results}
\begin{tabular}{@{}l l c c c c@{}}
\hline
Dataset & Model & Acc (\%) & k-mean & FLOPs & Latency (ms) \\
\hline
\multirow{3}{*}{ARC-Easy}
  & Base OLMoE & 77.82 & 8     & 5.37E8 & 11.203 \\
  & Elbow-8    & \textbf{78.24} & 7.623 & \textbf{4.86E8} & \textbf{10.497} \\
\hline
\multirow{3}{*}{ARC-Challenge}
  & Base OLMoE & 61.86 & 8     & 5.37E8 & 11.531 \\
  & Elbow-8    & \textbf{62.29} & 7.634 & \textbf{4.88E8} & \textbf{10.974} \\
\hline
\multirow{3}{*}{MMLU}
  & Base OLMoE & 51.74 & 8     & 5.37E8 & 13.287 \\
  & Elbow-8    & \textbf{51.75} & 7.613 & \textbf{4.91E8} & \textbf{12.571} \\
\hline
\multirow{3}{*}{HellaSwag}
  & Base OLMoE & \textbf{47.02} & 8     & 5.37E8 & 14.630 \\
  & Elbow-8    & 46.86 & 7.655 & \textbf{4.92E8} & \textbf{13.960} \\
\hline
\multirow{3}{*}{PIQA}
  & Base OLMoE & \textbf{73.18} & 8     & 5.37E8 & 13.398 \\
  & Elbow-8    & 72.85 & 7.589 & \textbf{4.85E8} & \textbf{12.944} \\
\hline
\multirow{3}{*}{WinoGrande}
  & Base OLMoE & \textbf{50.36} & 8     & 5.37E8 & 9.659 \\
  & Elbow-8    & 50.28 & 7.576 & \textbf{4.81E8} & \textbf{8.948} \\
\hline
\end{tabular}
\label{tab:eval_results}
\vspace{\baselineskip}
\end{wraptable}

\section{Experimental Validation}

We evaluate elbow-based routing on OLMoE under baseline top-$8$ routing and elbow-based routing capped at $k\le 8$ (Elbow-$8$) across six standard multiple-choice benchmarks: MMLU, ARC-Easy, ARC-Challenge, HellaSwag, PIQA, and WinoGrande. Table~\ref{tab:eval_results} reports accuracy, the mean number of active experts per token ($k$-mean), estimated FLOPs, and wall-clock latency per forward pass. See ~\ref{implementation} for additional implementation details. Across datasets, Elbow-$8$ matches baseline accuracy (within $\le 0.33$ points) while reducing average latency by $5.3\%$ on our setup and decreasing estimated compute, with an overall $k$-mean of $7.615$. These results indicate that elbow-based routing provides a lightweight inference-time efficiency improvement without retraining or changes to model weights. 
\section{Conclusion}

We introduced elbow-based routing, a training-free inference-time plugin that selects a token-specific number of experts by detecting an elbow in the router’s sorted probability curve. On OLMoE with a cap of $K=8$, this simple monotone pruning rule reduces average latency by $5.3\%$ across six benchmarks while preserving accuracy, without retraining, architectural changes, or additional hyperparameters. Across over 2M routing decisions, we observe that router distributions exhibit sharp and consistent elbows across layers and datasets, supporting a head--tail structure in which a small set of experts carries most routing signal. We further provide load-balance guarantees and empirical evidence that the induced changes in expert utilization are small, suggesting that routers trained with fixed top-$K$ often contain enough structure to enable lightweight test-time efficiency improvements and serve as a diagnostic lens for routing behavior.

\bibliography{iclr2026_conference}
\bibliographystyle{iclr2026_conference}

\appendix
\section{Appendix}
\subsection{Algorithm for Determining Elbows}
Below is an approximation of the Kneedle algorithm used to determine elbow indices: 
\begin{algorithm}[H]
\caption{Elbow Detection via Kneedle Approximation}
\label{alg:elbow}
\begin{algorithmic}
\Require Router logits $l \in \mathbb{R}^{N}$ over $N$ experts
\State $p \gets \mathrm{softmax}(l)$
\State $p \gets \mathrm{sort}_{\downarrow}(p)$ \Comment{sort probabilities in descending order}
\State $p_{\min} \gets \min(p)$; \quad $p_{\max} \gets \max(p)$
\For{$i = 0$ to $N-1$}
    \Comment{normalize to unit square and compute deviation}
    \State $p'_i \gets \dfrac{p_i - p_{\min}}{p_{\max} - p_{\min} + \epsilon}$
    \State $x'_i \gets \dfrac{i}{N-1}$
    \State $D_i \gets p'_i - x'_i$
\EndFor
\State $e \gets \arg\max_{i \in \{0,\dots,N-1\}} D_i$
\State \textbf{return} $k \gets e + 1$
\end{algorithmic}
\end{algorithm}

\newpage
\subsection{Additional Elbow Characterization}
\subsubsection{Sample Elbow Logits}
\begin{figure}[h]
    \centering
    \includegraphics[width=0.75\linewidth]{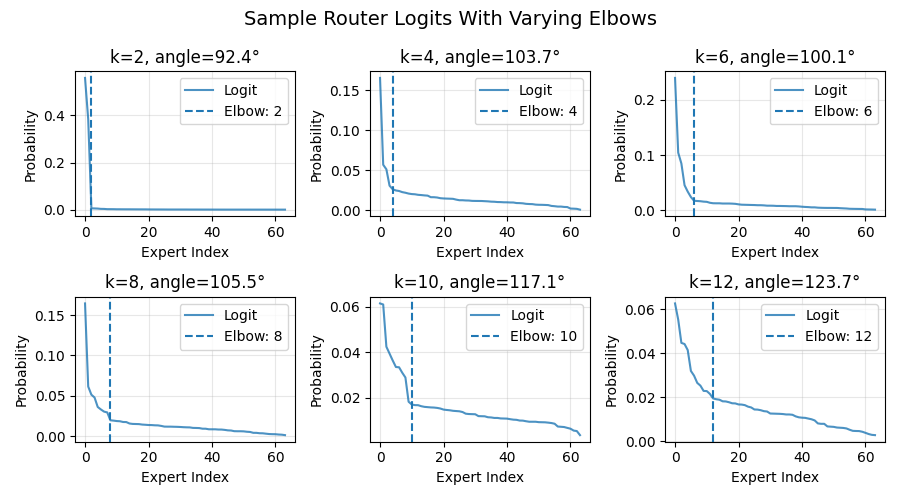}
    \caption{Plots of sorted router probabilities with elbow indexes at $k=2,4,6,8,10,12$ and their respective elbow angles. Our implementation of the Kneedle algorithm effectively identifies the 'elbow' of sorted router probability curves.}
    \label{fig:elbowlogitsample}
\end{figure}

\subsubsection{Elbow Analysis by MMLU Subject}
\begin{figure}[h]
\centerline{\includegraphics[width=\textwidth]{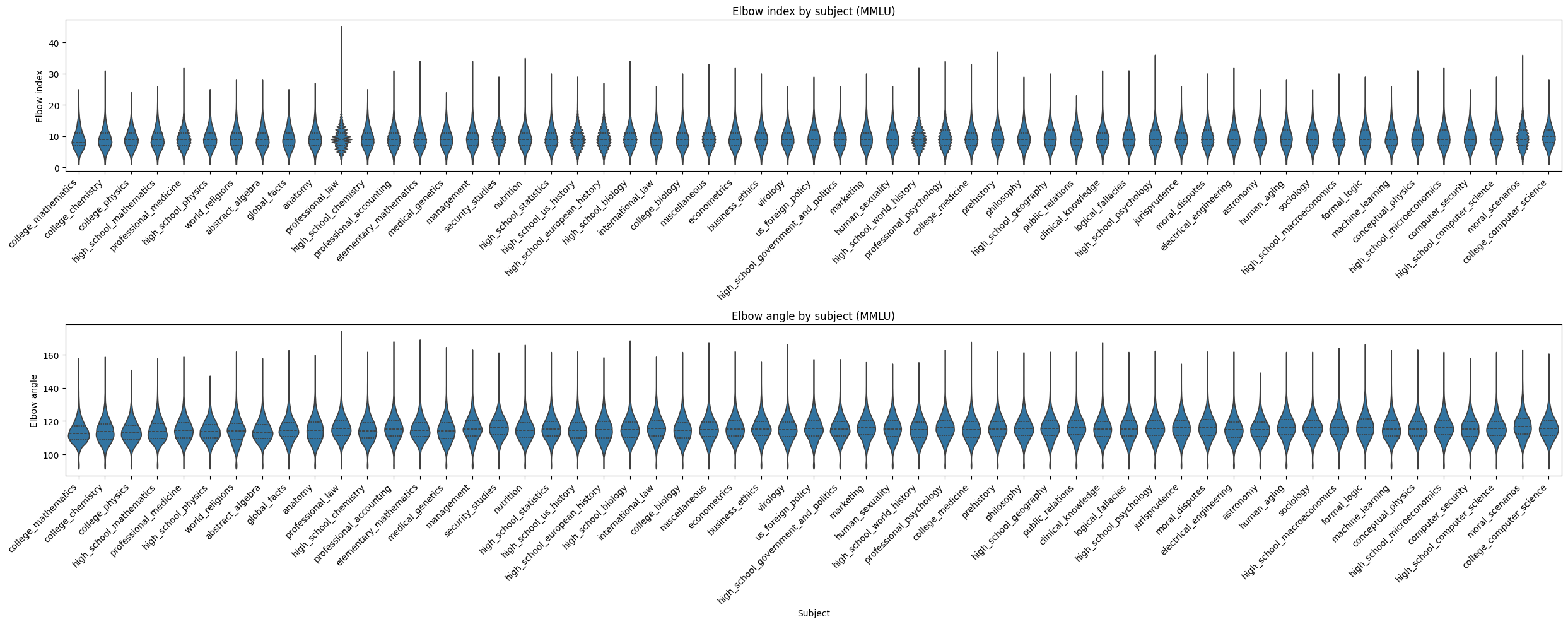}}
\caption{\textbf{Elbow angle and index vary little across subjects.} We plot the distribution of elbow index and elbow angle by subject above and conduct a one-way ANOVA test and find $p<0.001, \eta^2 = 0.02$ for elbow indices and $p<0.001,\eta^2=0.01$ with $N=1,704,247$. This indicates that a very small fraction of variance is attributed to subject.}
\label{supp1}
\end{figure}

\subsubsection{Elbow Angle and Elbow Index Correlation}
\begin{figure}[h]
    \centering
    \includegraphics[width=0.5\linewidth]{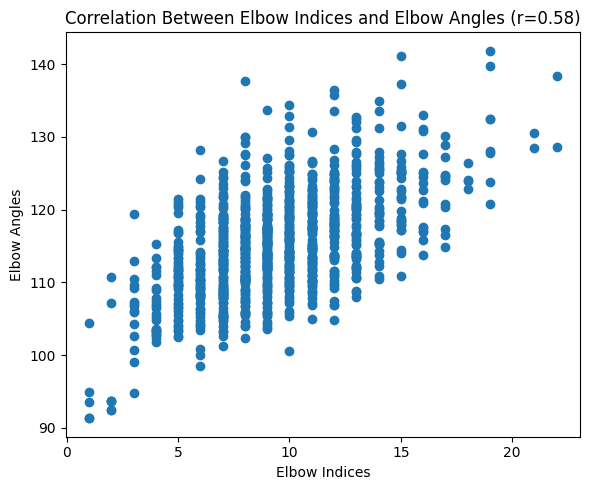}
    \caption{\textbf{Elbow indexes are correlated with elbow angle.} Elbow angle increases with elbow index, indicating that tokens requiring more active experts tend to exhibit less sharp transitions in the sorted router distribution (Pearson $r=0.58$)}
    \label{fig:elbowangle}
\end{figure}

\newpage
\subsection{Additional Load-Balance Guarantees Details}
\label{app:load_balance}

We analyze the impact of elbow-$K$ routing on expert load balance relative to standard top-$K$ routing. Elbow-$K$ routing acts as a monotone pruning of top-$K$ routing: for each token, it selects a subset of the experts chosen by top-$K$. Under this structure, we show that changes in distribution of the load balance are bounded and minimal.

\textbf{Notation.} Let $\mathcal{S}_{\mathrm{top}\text{-}K}(x)$ and $\mathcal{S}_{\mathrm{elb}\text{-}K}(x)$ denote the
experts selected by standard top-$K$ routing and elbow-$K$ routing, respectively, for token $x$. Let $k(x) = |\mathcal{S}_{\mathrm{elb}\text{-}K}(x)|$ denote the number of experts selected by
elbow-$K$ routing for token $x$, and define the expected fraction of pruned assignments as
\[
\delta = 1 - \frac{\mathbb{E}[k(x)]}{K}.
\]

For a routing rule $r \in \{\mathrm{top}\text{-}K,\,\mathrm{elb}\text{-}K\}$, define the expected
per-expert load
\[
L_i^{(r)} = \mathbb{E}_{x \sim \mathcal{D}}\!\left[\mathbf{1}\{ i \in \mathcal{S}_r(x)\}\right],
\]
and the corresponding normalized utilization distribution
\[
q_i^{(r)} = \frac{L_i^{(r)}}{\sum_{j=1}^N L_j^{(r)}},
\qquad
q^{(r)} \in \Delta^{N-1}.
\]

\begin{theorem}[Distribution Changes in Expert Utilization under Elbow-$K$ Routing]
\label{thm:utilization_stability}
For any MoE layer, let $q^{(\mathrm{top}\text{-}K)}$ and $q^{(\mathrm{elb}\text{-}K)}$ denote the
normalized per-expert utilization distributions under top-$K$ and elbow-$K$ routing, respectively.
Then, we claim that
\[
\bigl\| q^{(\mathrm{elb}\text{-}K)} - q^{(\mathrm{top}\text{-}K)} \bigr\|_1 \le  \frac{2\delta}{1-\delta}.
\]
\end{theorem}

\begin{proof}
By construction, $\mathcal{S}_{\mathrm{elb}\text{-}K}(x) \subseteq \mathcal{S}_{\mathrm{top}\text{-}K}(x)$
for all tokens $x$. Therefore, working at the vector level, define the removed-mass vector $R = L^{(\mathrm{top}\text{-}K)} - L^{(\mathrm{elb}\text{-}K)} \in \mathbb{R}^N_+.$

Since top-$K$ selects exactly $K$ experts per token and elbow-$K$ selects $k(x)$ experts, we have
\[
\sum_{i=1}^N L_i^{(\mathrm{top}\text{-}K)} = K,
\qquad
\sum_{i=1}^N L_i^{(\mathrm{elb}\text{-}K)} = \mathbb{E}[k(x)] = (1-\delta)K, \qquad  \sum_{i=1}^N R_i = \delta K.
\]

Passing to the corresponding normalized removed-mass distribution we have
\[
r = \frac{R}{\delta K},
\qquad r \in \Delta^{N-1},
\]
which combined with  the definitions of the normalized utilization distributions gives
\[
q^{(\mathrm{elb}\text{-}K)}
= \frac{L^{(\mathrm{elb}\text{-}K)}}{(1-\delta)K}
= \frac{L^{(\mathrm{top}\text{-}K)} - R}{(1-\delta)K} = \frac{q^{(\mathrm{top}\text{-}K)}K -R}{(1-\delta)K} = \frac{q^{(\mathrm{top}\text{-}K)} -r\delta}{1-\delta} \implies
\]
\[
\implies q^{(\mathrm{elb}\text{-}K)} - q^{(\mathrm{top}\text{-}K)}
=
\frac{\delta}{1-\delta}\bigl(q^{(\mathrm{top}\text{-}K)} - r\bigr).
\]

Since $q^{(\mathrm{top}\text{-}K)}$ and $r$ are probability distributions, we have that worst-case
$\|q^{(\mathrm{top}\text{-}K)} - r\|_1 \le \sum_{i=1}^N |q^{(\mathrm{top}\text{-}K)}_i - r_i| \le \sum_{i=1}^N q^{(\mathrm{top}\text{-}K)}_i + \sum_{i=1}^N  r_i \le 2$, which gives
\[
\bigl\| q^{(\mathrm{elb}\text{-}K)} - q^{(\mathrm{top}\text{-}K)} \bigr\|_1
\le
\frac{2\delta}{1-\delta}.
\]
\end{proof}

\vspace{-8mm}
\textbf{Corollary 1.} The above bound implies that the change in the load of the most-utilized expert is
\[
\big|\|q^{(\mathrm{elb}-K)}\|_\infty - \|q^{(\mathrm{top}-K)}\|_\infty\big|
\le
\|q^{(\mathrm{elb}-K)} - q^{(\mathrm{top}-K)}\|_1
\le
\frac{2\delta}{1-\delta}.
\]
Thus, elbow-$K$ routing cannot substantially increase the load of the most-utilized expert, which is often a determinant of inference-time latency.

\textbf{Corollary 2.} The above bound implies that the coefficient of variation $\mathrm{CV}^2(q) = N\| q\|_2^2-1$ change is
\[ 
\Big|\mathrm{CV}^2(q^{(\mathrm{elb}-K)})-\mathrm{CV}^2(q^{(\mathrm{top\text{-}K})})\Big| 
=
N \left| \, \|q \,^{{(\mathrm{elb}-K)}}\|_2^2 - \|q \,^{(\mathrm{top-}K)}\|^2_2\right| 
\le 
\] 
\[ 
=
\frac{1}{(1-\delta)^2}
\Bigl|
(2\delta-\delta^2)\|q^{(\mathrm{top}\text{-}K)}\|_2^2
- 2\delta\langle q^{(\mathrm{top}\text{-}K)}, r\rangle
+ \delta^2\|r\|_2^2
\Bigr|  \le \frac{2N\delta}{(1-\delta)^2},
\]
where we used that $0\le \|q^{(\mathrm{top}\text{-}K)}\|_2^2\le 1$,
$0\le \|r\|_2^2\le 1$,
and $0\le \langle q^{(\mathrm{top}\text{-}K)},r\rangle\le 1$ since $q^{(\mathrm{top}\text{-}K)},r\in\Delta^{N-1}$ to get $0\le (2\delta-\delta^2)\|q^{(\mathrm{top}\text{-}K)}\|_2^2+ \delta^2\|r\|_2^2\le 2 \delta$.

The alignment term $\langle q^{(\mathrm{top}\text{-}K)}, r\rangle$ captures whether pruning removes
assignments primarily from heavily utilized experts as opposed to low-utilized ones. Pruning aligned with high-load experts tends to
reduce $\mathrm{CV}^2$, while the opposite pattern can increase it. \\

\textbf{Discussion.} As we have mentioned in Section~\ref{sec:load_balance_guarantees}, empirically we observed across the six benchmarks that $\delta = 1 - \frac{\mathbb{E}[k(x)]}{K} =0.05$.  

Using this value, Theorem~1 guarantees that the normalized expert utilization distribution changes by at most
\[
\|q^{(\mathrm{elb\text{-}K})} - q^{(\mathrm{top\text{-}K})}\|_1
\;\le\;
\frac{2\delta}{1-\delta}
\;\approx\;
0.10.
\]
This tells us that after pruning, no more than worst-case 5\% of the total normalized expert utilization mass can be redistributed across experts. Empirically, the measured $\ell_1$ changes in normalized utilization are around $0.05$ or less across layers (Table~\ref{loadbalance}), indicating
that the relative expert usage profiles are preserved even more closely than required by the guarantee.

Similarly, Corollary~2 bounds the change in squared coefficient of
variation as
\[
\bigl|CV^2(q^{(\mathrm{elb\text{-}K})}) - CV^2(q^{(\mathrm{top\text{-}K})})\bigr|
\;\le\;
\frac{2N\delta}{(1-\delta)^2}
\;\approx\;
7.
\]
This is a rather loose bound. Empirically, the observed changes in CV vary from $1$ to $8\%$ across layers (Table~\ref{loadbalance}), indicating that elbow-based routing preserves not only the total utilization distribution but also its variability structure significantly more closely than required by the guarantee.

\subsection{Implementation Details}
\label{implementation}
We implemented elbow-based routing on OLMoE-1B-7B-0924-Instruct (\cite{olmoe}), a sparse MoE model with 64 experts and default top-8 routing. We select this model for its state-of-the-art performance, as well as its small model size given computational resource constraints. All evaluations were performed with an NVIDIA H200 GPU. We implement elbow-based routing via a runtime monkey patch that replaces \verb|OlmoeSparseMoeBlock.forward| with an instrumented variant, while caching the original method for reversible restoration. 

FLOPs are estimated analytically per MoE block as the sum of (i) router cost—one dense linear projection ($2NdE$), softmax ($5NE$), and a top-k/sort term ($NE\log_2 E$)—and (ii) expert MLP cost—up and down projections ($2N\bar{k} d m$) each plus activation ($N\bar{k} m$), where N=batch$\times$seq, d=hidden dim, m=4d, E=\# experts, and $\bar{k}$ is the observed mean selected experts per token. For dynamic routing we additionally account for elbow detection overhead as per-token sort ($NE\log_2 E$) plus linear-time normalization, subtraction, and argmax terms ($\approx (4+1+1)NE$).

Latency was measured with wall-clock timing around the MoE block forward pass using \verb|time.perf_counter()|, with an explicit \verb|torch.cuda.synchronize()| immediately before starting and after producing the final output tensor to ensure all queued CUDA kernels completed. 

All code is available at \url{https://github.com/rpan188/elbow-routing}. 

\end{document}